\documentclass{article}
\usepackage{spconf}
\usepackage{xparse}

\usepackage[ruled,vlined,linesnumbered]{algorithm2e}

\usepackage{amsmath}
\usepackage{amstext}
\usepackage{amssymb}
\usepackage{graphicx}
\usepackage{hyphenat}
 \usepackage{color}
 \usepackage{cite}
 
\usepackage[T1]{fontenc}
\usepackage[utf8]{inputenc}

\NewDocumentCommand\col{g}{%
  \IfNoValueTF{#1}{\ensuremath{\mathrm{vec}}}{\ensuremath{\mathrm{vec}}\of{#1}}%
}

\NewDocumentCommand\of{og}{%
  \IfNoValueTF{#1}%
    { \IfNoValueTF{#2}{}{\!\({#2}\)} }%
    { \IfNoValueTF{#2}{\!\[{#1}\]}{\!\{{#2}\}} }%
}

\DeclareMathOperator{\Diff}{\ipaclap{D}{\raisebox{.204em}{\textpalhook}\kern.44em}\kern-.1em}
\NewDocumentCommand\diff{g}{%
  \IfNoValueTF{#1}
  {\text{\texthtd}}
  {\text{\texthtd}\of{#1}}%
}

\RenewDocumentCommand\ln{g}{%
  \IfNoValueTF{#1}{\mathrm{ln\ }}{\mathrm{ln}\of{#1}}%
}

\NewDocumentCommand\Real{og}{%
  \IfNoValueTF{#1}%
    { \IfNoValueTF{#2}{\mathcal{R}\!\!\mathpzc{e}}{\mathcal{R}\!\!\mathpzc{e}\!\{{#2}\}} }%
    { \IfNoValueTF{#2}{\mathcal{R}\!\!\mathpzc{e}\!\[{#1}\]}{\mathcal{R}\!\!\mathpzc{e}\!\({#2}\)} }%
}
\NewDocumentCommand\Imag{og}{%
  \IfNoValueTF{#1}%
    { \IfNoValueTF{#2}{\mathcal{I}\!\!\mathpzc{m}}{\mathcal{I}\!\!\mathpzc{m}\!\{{#2}\}} }%
    { \IfNoValueTF{#2}{\mathcal{I}\!\!\mathpzc{m}\!\[{#1}\]}{\mathcal{I}\!\!\mathpzc{m}\!\({#2}\)} }%
}

\RenewDocumentCommand\cos{g}{%
  \IfNoValueTF{#1}{\mathrm{cos}}{\mathrm{cos}\of{#1}}%
}
\RenewDocumentCommand\sin{g}{%
  \IfNoValueTF{#1}{\mathrm{sin}}{\mathrm{sin}\of{#1}}%
}
\RenewDocumentCommand\tan{g}{%
  \IfNoValueTF{#1}{\mathrm{tan}}{\mathrm{tan}\of{#1}}%
}
\RenewDocumentCommand\arccos{g}{%
  \IfNoValueTF{#1}{\mathrm{arccos}}{\mathrm{arccos}\of{#1}}%
}
\RenewDocumentCommand\arcsin{g}{%
  \IfNoValueTF{#1}{\mathrm{arcsin}}{\mathrm{arcsin}\of{#1}}%
}
\RenewDocumentCommand\arctan{g}{%
  \IfNoValueTF{#1}{\mathrm{arctan}}{\mathrm{arctan}\of{#1}}%
}
\RenewDocumentCommand\cot{g}{%
  \IfNoValueTF{#1}{\mathrm{cot}}{\mathrm{cot}\of{#1}}%
}

\newcommand{\floor}[1]{\ensuremath{\left\lfloor #1 \right\rfloor}}

\NewDocumentCommand\tr{g}{%
  \IfNoValueTF{#1}{\mathrm{tr}}{\mathrm{tr}\of{#1}}%
}

\NewDocumentCommand\diag{og}{%
  \IfNoValueTF{#1}%
    { \IfNoValueTF{#2}{\ensuremath{\mathrm{diag}}}{\ensuremath{\mathrm{diag}\of{#2}}} }%
    { \IfNoValueTF{#2}{\ensuremath{\mathrm{diag}\of[#1]}}{\ensuremath{\mathrm{diag}\of[]{#2}}} }%
}

\RenewDocumentCommand\exp{g}{%
  \IfNoValueTF{#1}{\ensuremath{\mathrm{exp}}}{\ensuremath{\mathrm{exp}}\of{#1}}%
}

\newcommand{\E}[2][]{\Operator[#1]{E}{#2}}

\NewDocumentCommand\C{g}{%
  \IfNoValueTF{#1}{\mathrm{Cov}}{\mathrm{Cov}\of{#1}}%
}

\renewcommand{\(}{\ensuremath{\left(}}
\renewcommand{\)}{\ensuremath{\right)}}
\renewcommand{\[}{\ensuremath{\left[}}
\renewcommand{\]}{\ensuremath{\right]}}

\let\oldBracketLeft\{
\let\oldBracketRight\}
\renewcommand{\{}{\ensuremath{\left\oldBracketLeft}}
\renewcommand{\}}{\ensuremath{\right\oldBracketRight}}

\NewDocumentCommand\F{og}{%
  \IfNoValueTF{#1}%
    { \IfNoValueTF{#2}{\mathcal{F}}{\mathcal{F}\!\{{#2}\}} }%
    { \IfNoValueTF{#2}{\mathcal{F}\!\[{#1}\]}{\mathcal{F}\!\({#2}\)} }%
}
\NewDocumentCommand\FInv{og}{%
  \IfNoValueTF{#1}%
    { \IfNoValueTF{#2}{\mathcal{F}^{-1}}{\mathcal{F}^{-1}\!\{{#2}\}} }%
    { \IfNoValueTF{#2}{\mathcal{F}^{-1}\!\[{#1}\]}{\mathcal{F}^{-1}\!\({#2}\)} }%
}

\NewDocumentCommand\rect{g}{%
  \IfNoValueTF{#1}
  {\ensuremath{\mathrm{rect}}}
  {\ensuremath{\mathrm{rect}\of{#1}}}%
}

\NewDocumentCommand\sinc{g}{%
  \IfNoValueTF{#1}
  {\ensuremath{\mathrm{sinc}}}
  {\ensuremath{\mathrm{sinc}\of{#1}}}%
}

\NewDocumentCommand\supp{g}{%
  \IfNoValueTF{#1}
  {\ensuremath{\mathrm{supp}}}
  {\ensuremath{\mathrm{supp}\of{#1}}}%
}

\let\oldMathcal\mathcal
\renewcommand{\mathcal}[1]{\ensuremath{\oldMathcal{#1}}}

\makeatletter
\def\foreach#1#2#3{%
  \@test@foreach{#1}{#2}#3,\@end@token
}

\def\@swallow#1{}

\def\@test@foreach#1#2{%
  \@ifnextchar\@end@token%
    {\@swallow}%
    {\@foreach{#1}{#2}}%
}

\def\@foreach#1#2#3,#4\@end@token{%
  #1{#2}{#3}%
  \@test@foreach{#1}{#2}#4\@end@token%
}
\makeatother

\newtheorem{theorem}{Theorem}[section]

\newtheorem{remark}{Remark}

\newenvironment{proof}[1][Proof]{\begin{trivlist}
\item[\hskip \labelsep {\bfseries #1}]}{\end{trivlist}}
\newenvironment{definition}[1][Definition]{\begin{trivlist}
\item[\hskip \labelsep {\bfseries #1}]}{\end{trivlist}}

\def\l|{\left|}
\def\r|{\right|}
\def\l({\left(}
\def\r){\right)}
\def\l[{\left[}
\def\r]{\right]}

\renewcommand{\E}[1]{\mathbb{E}\left[ #1 \right]}
\newcommand{\half}{{\frac{1}{2}}}

\begin{document} 
\title{Rate-optimal Meta Learning of Classification Error}


\name{Morteza Noshad Iranzad and Alfred O. Hero III \thanks{This work was partially supported by a grant from ARO, number W911NF-15-1-0479.}}
\address{University of Michigan,
Electrical Engineering and Computer Science,
Ann Arbor, Michigan, U.S.A}


\maketitle

\begin{abstract}

Meta learning of optimal classifier error rates allows an experimenter to empirically estimate the intrinsic ability of any estimator to discriminate between two populations, circumventing the difficult problem of estimating the optimal Bayes classifier. To this end we propose a weighted nearest neighbor (WNN) graph estimator for a tight bound on the Bayes classification error; the Henze-Penrose (HP) divergence. Similar to recently proposed HP estimators \cite{berisha2016}, the proposed estimator is non-parametric and does not require density estimation. However, unlike previous approaches the proposed estimator is rate-optimal, i.e., its mean squared estimation error (MSEE) decays to zero at the fastest possible rate of $O(1/M+1/N)$ where $M,N$ are the sample sizes of the respective populations. We illustrate the proposed WNN meta estimator for several simulated and real data sets.

\end{abstract}
\section{Introduction}
The minimum classification error, also known as Bayes error, is the best (minimum) average probability of error that can be achieved by any binary classifier of a sample coming from one of two classes.
Although the Bayes error can be represented in the form of an integral for a simplest case of binary classification, analytical evaluation of this integral is often not feasible, even when densities are known \cite{berisha2016}. Based on this fact, several previous works have investigated upper and lower bounds on the Bayes error, which are easy to compute analytically. A Bound based on Chernoff $\alpha$-divergence has been proposed in \cite{chernoff}. Bhattacharya divergence, which is a special case of Chernoff $\alpha$-divergence for $\alpha=\half$, is used in a number of applications involving Bayes error including feature selection \cite{xuan2006,zhang2007,reyes2006}. 
Recently Berisha et al proposed tighter lower and upper bounds based on HP divergence with parameter $p$, where $p$ is the prior probability of class 1 \cite{berisha2016}. They proved that the bounds are tight for $p=1/2$. 

The problem of estimating the Bayes classification error directly, without the need to estimate the Bayes classifier function, is called the meta-learning problem. It is of crucial importance in reinforcement learning where an estimate of potential performance gains is used to guide the choice of future actions, e.g., selecting a data source \cite{sutton1}. Information divergence approaches to solving the meta-learning problem estimate various divergence functionals that measure the dissimilarity between the population distributions.

There are two major information divergence estimation approaches; plug-in and direct estimation. Plug-in methods first compute estimates of the population densities and plug them into the formula for the information divergence. Kernel Density Estimator (KDE) and $k$-Nearest Neighbor ($k$-NN) population density estimation are commonly used in the plug-in approach \cite{kevin2014,krishnamurthy2014}. In contrast, direct estimation approaches bypass density estimation entirely, producing an estimator of the information divergence using geometric functions of the data. Geometric quantities such as minimal graphs are commonly used for direct estimation of information divergences using the $k$-NN graph on the dataset to estimate R\'{e}nyi entropy \cite{beardwood1959shortest, hero2003}, using the MST graphs to estimate HP divergence \cite{Henze}, and using the nearest neighbor ratios (NNR) method to estimate various divergence measures\cite{noshad}, are some of the examples of direct graph based estimators. Direct estimation methods have several advantages over plug-in estimators such as lower computational complexity, simplicity of the estimator, imposing less constraints on the density functions, and offering an intuitive graph theoretical interpretation of the information measure. 

The HP divergence was defined by Henze \cite{Henze_MST,henze_knn} as the almost sure limit of the Friedman-Rafsky (FR) multi-variate two sample test statistic. Thus the FR two sample test statistic can be interpreted as an asymptotically consistent estimator of the HP divergence. The FR procedure is as follows. Assume that we have two data-sets $X$ and $Y$. The FR test statistic is formed by counting the edges of MST graph of the joint data set $Z:=X\cup Y$, which connect dichotomous points, i.e., a point in $X$ to a point in  $Y$. Later in \cite{henze_knn}, Henze proposed another similar graph based estimator that considers $k$-NN graph instead of the MST graph. However, the main FR test statistics using MST graph has received more attention than the $k$-NN variant.  
The authors of \cite{henze_knn} proved the asymptotic consistency and unbiasedness of FR statistics based on type coincidence, but the convergence rates of these estimators have remained unknown since then. In \cite{moon2015} Moon et al used an optimal plug-in density estimator to estimate HP divergence. As mentioned before, since plug-in estimation needs multi-step estimation procedure, which consists of estimating each of the densities in the first place, and then plugging-in the estimated densities in the divergence function, compared to the direct estimation methods, it suffers from slower runtimes and requires stringent on the density functions.

In this paper we propose a new direct estimator of the HP divergence based on a weighted $k$-NN graph. We first derive the convergence rates of the $k$-NN based FR test statistics, defined as the number of edges in the $k$-NN graph over the joint data set $Z:=X\cup Y$, which connect dichotomous points. We prove that the bias rate of this estimator is upper bounded by $O\of{\of{k/N}^{\gamma/d}}+e^{-ck}$, where $N$ and $d$ respectively are the number and dimension of the samples, $\gamma$ is the H\"{o}lder smoothness parameter of the densities and $c$ is a constant. Note that the convergence rate of this estimator worsens in higher dimensions and does not achieve the optimal parametric rate of $O\of{1/N}$. Therefore, we propose a direct estimation method based on a weighted $k$-NN graph. We refer to this method as the weighted nearest neighbor (WNN) estimator. The graph includes a weighted and directed edge between any pair of nodes $R$ and $S$ only if the types of $R$ and $S$ are different (i.e. $R\in X$ and $S\in Y$) where $S$ is the set of $k$th nearest neighbors of $R$. We prove that if the edge weights are obtained from the solution of a certain optimization problem, we can construct a rate-optimal HP divergence estimator based on the sum of the weights of the dichotomous edges. The convergence rate of this estimator is proved to be $O(1/N)$, which is both optimal and independent of $d$. Finally, we emphasize that the proposed WNN estimator is completely different from the weighted matching estimator.

\section{Main Results}

In this section we recall the Henze-Penrose (HP) divergence and propose an optimal estimator based on the $k$-NN graph. All of the proofs of the convergence theorems are provided in the Appendix of arXiv version.

Consider two density functions $f_X$ and $f_Y$ with support $\mathcal{M}\subseteq \mathbb{R}^d$. The HP-divergence between $f_X$ and $f_Y$ is denoted by $D_{P}\left(f_X(x)\vert\vert f_Y(x)\right)$ and defined as

\begin{align}\label{HP_def}
D_{p}=1-\int \frac{f_X(x)f_Y(x)}{pf_X(x)+qf_Y(x)}dx
\end{align}
where $p$ is a parameter and $p+q=1$. We also define $\eta:=p/q$. In \cite{noshad} $p$ is the number of empirical samples from the first class.

\textbf{Assumptions:}
We assume that the densities $f_1$ and $f_2$ have the same bounded support set and are lower bounded by $C_L>0$ and upper bounded by $C_U$. We also assume that they belong to H\"{o}lder smoothness class with parameter $\gamma$:


\begin{definition}\label{Holder}
Given the support set $\mathcal{X} \subseteq \mathbb{R}^d$, a function $f:\mathcal{X} \to \mathbb{R}$ is called H\"{o}lder continuous with parameter $0<\gamma\leq 1$, if there exists a positive constant $G_f$, possibly depending on $f$, such that 
\begin{equation}
|f(y)-f(x)|\leq G_f\|y-x\|^{\gamma},
\end{equation}
for every $x\neq y \in \mathcal{X}$.
\end{definition}




\subsection{$k$-NN Estimator}

\begin{definition}
Let $X=\{X_1,...,X_N\}$ and $Y=\{Y_1,...,Y_M\}$ respectively denote i.i.d samples with densities $f_1$ and $f_2$, such that $M=\floor{\frac{Nq}{p}}$.  
Let $G_k(X,Y)$ be the graph of $k$ nearest neighbors constructed over the joint set $Z = X\cup Y$. In other words, edges of $G_k(X,Y)$ connect the points $x\in Z$ to their $k$th nearest neighbors. Assume that $\mathcal{E}(X,Y)$ is the set of edges of $G_k(X,Y)$ connecting dichotomous points. Then the $K$-NN estimator of HP-divergence, $\widehat{D_p}$, is defined as 

\begin{align}\label{define_knn_est}
\widehat{D_p}(X,Y)=1-|\mathcal{E}(X,Y)|\frac{N+M}{2NM}.
\end{align}

The idea behind this estimator is similar to the idea of MST estimator of HP-divergence proposed by Friedman and Rafsky (FR) \cite{FR}, in which we count the number of edges connecting dichotomous points in the minimal spanning tree of the joint data \cite{Henze_MST}. If $N=M$ and the densities are almost equal, then with probability of almost $1/2$ every $k$th nearest neighbor edge belongs to \mathcal{E}(X,Y). In this case $|\mathcal{E}(X,Y)|\approx N$, and $\widehat{D_p}\approx 0$. 

In the following theorems we derive upper bounds on the bias and variance rates. Here the bias and variance are defined as $\mathbb{B}[\hat{T}]=\mathbb{E}[\hat{T}]-T$ and $\mathbb{V}[\hat{T}]=\mathbb{E}[\hat{T}^2]-\mathbb{E}[\hat{T}]^2$, respectively, where $\hat{T}$ is an estimator of the parameter $T$.

\begin{theorem} \label{bias_theorem}
 The bias of the $k$-NN estimator for HP divergence satisfies
\begin{align} \label{bias_knn}
\mathbb{B}\of[\widehat{D_p}(X,Y)]= O\of{(k/N)^{\gamma/d}}+O\of{\mathcal{C}(k)},
\end{align}
where $\mathcal{C}(k):=exp(-3k^{1-\delta})$ for a fixed $\delta\in (2/3,1)$.
Here $\gamma$ is the H\"{o}lder smoothness parameter. 
\end{theorem}

\begin{remark}
Note that in order that $\hat{D}_p(X,Y)$ be asymptotically unbiased, $k$ needs to grow with $N$. The minimum bias rate of $O\of{\frac{\log N}{N}}^{\gamma/d}$ can be achieved by selecting $k=O(\log N)$.
\end{remark}

\begin{theorem}\label{variance}
The variance of the $k$-NN estimator for the HP divergence satisfies
\begin{align}
\mathbb{V}\of[\widehat{D_p}(X,Y)]\leq O\of{\frac{1}{N}}.
\end{align}
\end{theorem}

\begin{algorithm} \label{algo_1}
\DontPrintSemicolon
\SetKwInOut{Input}{Input}\SetKwInOut{Output}{Output}
\Input{Data sets $X=\{X_1,...,X_N\}$, $Y=\{Y_1,...,Y_M\}$}

\BlankLine
 
$Z\leftarrow X \cup Y$    \;
\For {each point $Z_i$ in $Z$}{
		
        If ($Z_i\in X$ and $Q_k(Z_i)\in Y$) \\
        \qquad or ($Z_i\in Y$ and $Q_k(Z_i)\in X$) \\
        \qquad\qquad then $S\leftarrow S+1 \qquad$}

\Output{$1-S\frac{N+M}{2NM}$}

\caption{$k$-NN Estimator of HP Divergence }
\end{algorithm}

\subsection{WNN Estimator}
Note that the bias term in Theorem \ref{bias_theorem} depends on $d$. Therefore, for higher dimensions the estimator convergence rate is slower. In order to resolve this issue and achieve optimum convergence rate in any dimension, we propose a modified $k$-NN graph based estimator of HP divergence.  
Assume that the density functions are in the H\"{o}lder space $\Sigma(\gamma,B)$, which consists of functions on $\mathcal{X}$ continuous derivatives up to order $q=\floor{\gamma}\geq d$  and the $q$th partial derivatives are H\"{o}lder continuous with exponent $\gamma'=:\gamma-q$ and Lipshitz constant $B$. Further, assume that the density derivatives up to order $d$ vanish at the boundary.
Fix a constant $L$ where $L\geq d$. Let $\mathcal{L}:=\{l_1,...,l_L\}$ be a set of index values with $l_i<c$, where $\kappa=\floor{c\sqrt{N}}$. We further define the value of the $k$-NN parameter as a function of $l$, i.e. $K(l):=\floor{l\sqrt{N}}$.  

\begin{definition}
Let $X=\{X_1,...,X_N\}$ and $Y=\{Y_1,...,Y_M\}$ respectively denote i.i.d samples with densities $f_1$ and $f_2$, such that $M=\floor{\frac{Nq}{p}}$.  
Let the weight vector $W:=[W(l_1), W(l_2),..., W(l_L)]$ be the solution to the following optimization problem:
\begin{align}
\min_w &\qquad \|w\|_2 \nonumber\\
\textit{subject to} &\qquad \sum_{l\in \mathcal{L}}w(l)=1, \nonumber\\
&\qquad \sum_{l\in \mathcal{L}}w(l)l^{i/d}=0, i\in \mathbb{N}, i\leq d.
\end{align} 

Now define $G_K^W(X,Y)$ as a weighted directed graph over vertices of the joint set $X\cup Y$. There is a directed edge with the weight $W(l)$ between any pair of nodes $R$ and $S$, only if the types of $R$ and $S$ are different (i.e. $R\in X$ and $S\in Y$), where $S$ is the $K(l)$-th nearest neighbor of $R$ for some $l\in \mathcal{L}$. We represent the set of edges of $G_K^W(X,Y)$ by $\mathcal{E}_K^W(X,Y)$.

The proposed WNN estimator   $\widehat{D}_p^W$ of HP divergence, is defined as 
\begin{align}
\widehat{D}_p^W(X,Y)=1-|\mathcal{E}_K^W(X,Y)|\frac{N+M}{2NM}.
\end{align}
\end{definition}

\begin{theorem}\label{WNN_Theorem}
The MSE of the WNN estimator is $O(1/N)$.

\end{theorem}

\begin{algorithm} \label{algo_1}
\DontPrintSemicolon
\SetKwInOut{Input}{Input}\SetKwInOut{Output}{Output}
\Input{Data sets $X=\{X_1,...,X_N\}$, $Y=\{Y_1,...,Y_M\}$}

\BlankLine
 
$Z\leftarrow X \cup Y$    \;

\For {$l\in \mathcal{L}$}{
\For {each point $Z_i$ in $Z$}{
		
        If ($Z_i\in X$ and $Q_l(Z_i)\in Y$) \\
        \qquad or ($Z_i\in Y$ and $Q_l(Z_i)\in X$) \\
        \qquad\qquad then $S\leftarrow S+W(l)\qquad$}
} 

\Output{$1-S\frac{N+M}{2NM}$}

\caption{WNN Estimator of HP Divergence }
\end{algorithm}

\end{definition} 
\section{numerical Results}

In this section we investigate the behavior of the proposed estimator by numerical experiments. 

Fig. \ref{k} shows the mean estimated HP divergence between two truncated Normal RVs with mean vectors $[0, 0]$ and $[0, 1]$ and variance of $\sigma^2_1=\sigma^2_2=I_2$, as a function of number of samples, $N$, where $I_d$ is the identity matrix of size $d$. Three different values of $k$ are investigated. For each case we repeat the experiment $100$ times, and compute the bias and variance. As $N$ increases, the bias for any $k$ tends to zero. The experiments show that the bias decreases slower as $k$ increases, which is due to the $\of{\frac{k}{N}}^{\gamma/d}$ term in \eqref{bias_knn}. However, according to this experiment, the variance is almost independent of $k$ and decreases linear towards zero.   
\begin{figure}
	\centering
	\includegraphics[width=9cm]{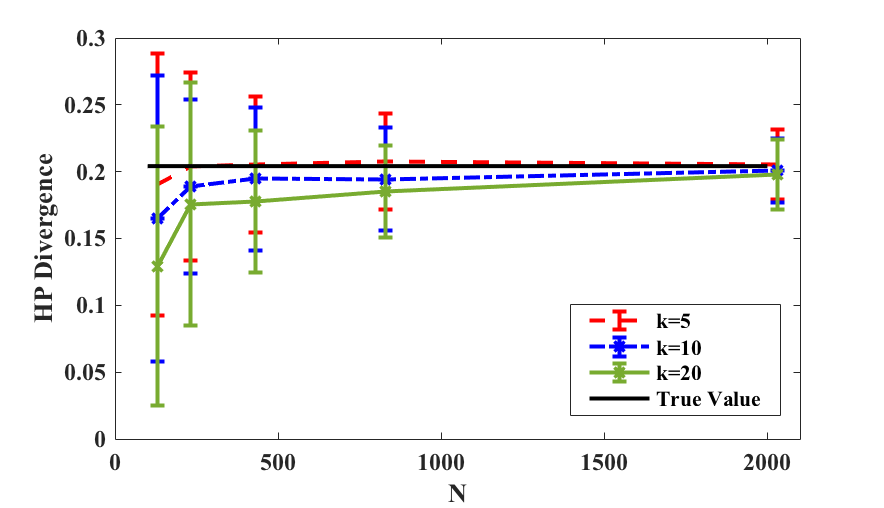}
	\caption{Comparison of the estimated values of $k$-NN estimator with $k=5,10,20$ for HP divergence between two truncated Normal RVs with mean vectors $[0, 0]$ and $[0, 1]$ and variances of $\sigma^2_1=\sigma^2_2=I_2$, plotted against $N$, the number of samples.}\label{k}
\end{figure}


Fig. \ref{dimension} shows the MSE of the $k$-NN estimator for HP divergence between two zero mean Normal random vectors in $\mathbb{R}^2$, with identical covariance matrix $I_d$ whose values are truncated within the range $x\in [-5, 5]$ and $y\in [-5,5]$. The experiment is repeated for three different dimensions of $d=2,10,20$ for fixed $k=5$. 
In agreement with our bias bound in \eqref{bias_knn}, as $d$ increases, the experiment shows that the MSE rate increases.
\begin{figure}
	\centering
	\includegraphics[width=9.0cm]{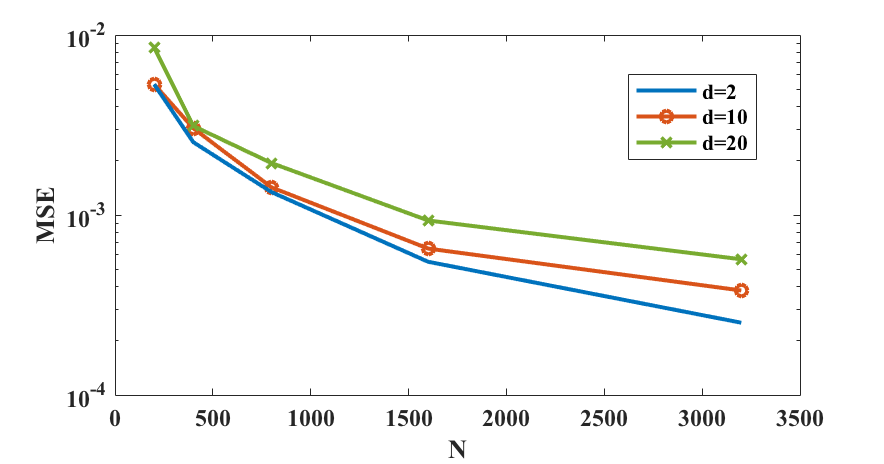}
	\caption{MSE of the $k$-NN estimator for HP divergence between two identical, independent and truncated Normal RVs, as a function of $N$.}\label{dimension}
\end{figure}


In Fig. \ref{d2_MSE} we compare the MSE rates of the three graph theoretical estimators of HP divergence; MST, $k$-NN, and WNN estimators. The HP divergence between two truncated Normal random variables with $d=2$, means of $\mu_1=[0,0]$, $\mu_2=[1,0]$, and covariance matrices of $\sigma_1=I_2$ and $\sigma_2=2I_2$. This experiment verifies the advantage of WNN estimator over the $k$-NN and MST estimators, in terms of their convergence rates. Also the performance of MST estimator is slightly better than the $k$-NN estimator. Note that in this experiment we have set the number of neighbors of the $k$-NN to $k=5$. 
\begin{figure}
	\centering
	\includegraphics[width=9cm]{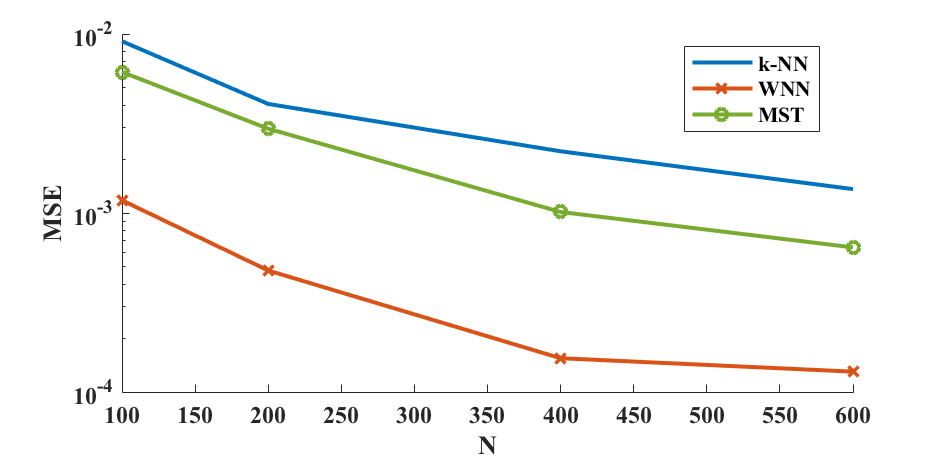}
	\caption{MSE comparison of the three graph theoretical estimators of HP divergence; MST, $k$-NN, and WNN estimators. }\label{d2_MSE}
\end{figure}

Fig. \ref{d2_CB} shows the comparison of the estimators of HP divergence between a truncated Normal RV with mean $[0,0]$ and covariance matrix of $I_2$, and uniform RV within $[-5,5]\times [-5,5]$, in terms of their mean value and $\%95$ confidence band. The confidence band becomes narrower for greater values of $N$, and the WNN estimator has the narrowest confidence band.

\begin{figure}
	\centering
	\includegraphics[width=8cm]{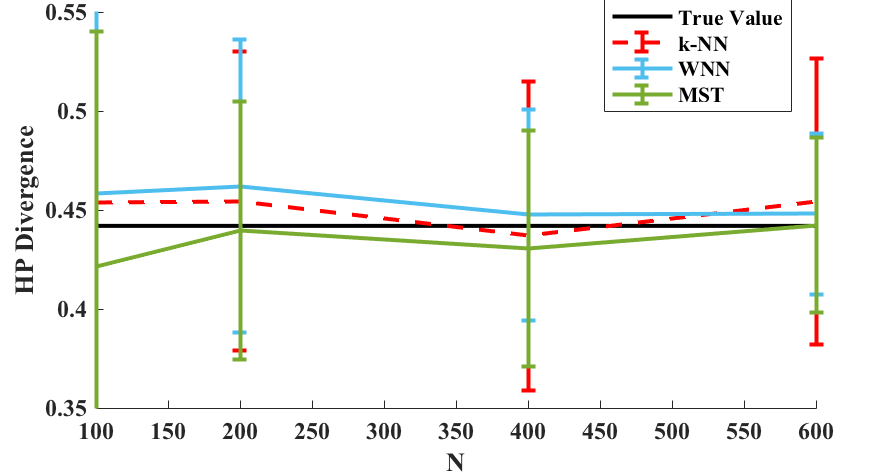}
    \caption{Comparison $k$-NN, MST and WNN estimators of HP divergence between a truncated Normal RV and a uniform RV, in terms of their mean value and $\%95$ confidence band. }\label{d2_CB}
\end{figure}


Finally in Fig. \ref{robot}, we 
compare performance of the WNN to that of the $k$-NN estimators with $k=5$ and $k=10$, for a real data set \cite{freire2009short, robot_data}. The data are measurement from a set of ultrasound sensors arranged circularly around a robot, which navigates through the room following the wall in a clockwise direction. There are total number of $5456$ instances (corresponding to different timestamps), and we use the information of four main sensors as the feature space. The instances are associated with four different classes of actions: move-forward, sharp-right-turn, slight-right-turn and turn-left. In Fig. \ref{robot} we consider the divergence between the sensor measurement for sharp-right-turn and move-forward classes. Note the superior performance of the WNN estimator as compared to the $k$-NN estimators.   
\begin{figure}
	\centering
\includegraphics[width=9cm]{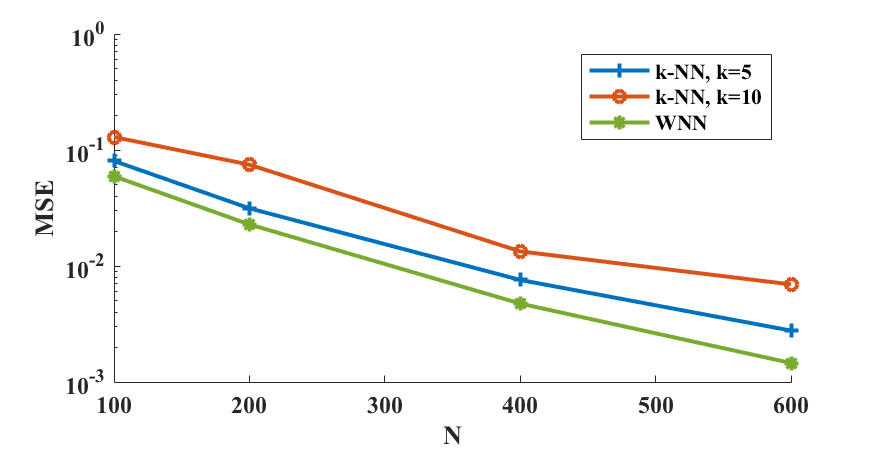}
    \caption{MSE Comparison of the WNN and $k$-NN  estimator with two different parameters $k=5$ and $k=10$ for the robot navigation. dataset}\label{robot}
\end{figure}

\section{Conclusion}
In this paper we derived the convergence rates of $k$-NN version of FR test statistics and  proposed an optimum direct estimation method for HP divergence, based on the weighted $k$-NN graph. We proved that WNN estimator can achieve optimum parametric MSE rate of $O(1/N)$, and we validated our results on simulated and real data sets.


\bibliographystyle{IEEEbib}
\bibliography{citations}
\onecolumn
\newpage

\section{Proofs}\label{Proof_Section}

In this section we derive an upper bound on the bias and variance of the $k$-NN and WNN estimators. In the following we first prove Theorem \ref{bias_theorem}:

\begin{proof}[Proof of Theorem \ref{bias_theorem}]
Let $\mathcal{E}_1(X,Y)$ and $\mathcal{E}_2(X,Y)$ respectively denote the set of $k$-NN edges of $X$ and $Y$ nodes, connecting dichotomous points, i.e., points in $X$ to $Y$ and points in $Y$ to $X$. So, we have $|\mathcal{E}(X,Y)|=|\mathcal{E}_1(X,Y)|+|\mathcal{E}_2(X,Y)|$. From the definition of the estimator, we have

\begin{align}\label{bias_proof_1}
&\mathbb{E}\of[{\widehat{D_p}(X,Y)}]\nonumber\\
&\qquad=1-\E{|\mathcal{E}(X,Y)|\frac{N+M}{2NM}}\nonumber\\
&\qquad=1-\E{|\mathcal{E}_1(X,Y)|\frac{N+M}{2NM}+|\mathcal{E}_2(X,Y)|\frac{N+M}{2NM}}\nonumber\\
&\qquad=1-\frac{\E{|\mathcal{E}_1(X,Y)|}}{2Nq}-\frac{\E{|\mathcal{E}_2(X,Y)|}}{2Mp}+O\of{\frac{1}{N}}\nonumber\\
&\qquad=1-\frac{\E{\sum_{i=1}^{N}N_i}}{2Nq}-\frac{\E{\sum_{i=1}^{M}M_i}}{2Mp}+O\of{\frac{1}{N}},
\end{align}
where $N_i$ and $M_i$ are defined as follows:

\begin{align}
N_i=
\begin{cases}
0 \qquad\qquad  Q_k(X_i) \in X \\
1 \qquad\qquad  Q_k(X_i) \in Y, 
\end{cases}
\end{align}
and 
\begin{align}
M_i=
\begin{cases}
0 \qquad\qquad  Q_k(Y_i) \in Y \\
1 \qquad\qquad  Q_k(Y_i) \in X. 
\end{cases}
\end{align}

The first term in RHS of \ref{bias_proof_1} can be simplified as

\begin{align}\label{bias_proof_2}
\frac{\E{\sum_{i=1}^{N}N_i}}{2Nq}&=\frac{\sum_{i=1}^{N}\E{N_i}}{2Nq}\nonumber\\
&=\frac{\mathbb{E}_{X_1\sim f(x)}\of[\E{N_1\vert X_1}]}{2q}\nonumber\\
&=\frac{\mathbb{E}_{X_1\sim f(x)}\of[\Pr\of{Q_k(X_1)\in Y}]}{2q}.
\end{align}

Now using a result from \cite{noshad} (Lemma 3.3), we have

\begin{align}\label{point_probability_equation}
\Pr\of{Q_i(X_1)\in Y} &= \frac{f_Y(X_1)}{f_X(X_1)+\eta f_Y(X_1)} +\epsilon_{\gamma,k},
\end{align}
where $\epsilon_{\gamma,k}:=O\of{(k/N)^{\gamma/d}}+O\of{\mathcal{C}(k)}$. So, \eqref{bias_proof_2} can further be simplified as 

\begin{align}\label{bias_proof_2}
\frac{\E{\sum_{i=1}^{N}N_i}}{2Nq}&=\frac{1}{2q} \int_x\frac{f_Y(x)}{\eta ^{-1}f_X(x)+f_Y(x)}f_X(x)dx + \epsilon_{\gamma,k} \nonumber\\
&=\half\int\frac{f_X(x)f_Y(x)}{pf_X(x)+qf_Y(x)}dx + \epsilon_{\gamma,k}.
\end{align}

Similarly, we get
\begin{align}\label{bias_proof_3}
\frac{\E{\sum_{i=1}^{M}M_i}}{2Mp}
&=\half\int\frac{f_X(x)f_Y(x)}{pf_X(x)+qf_Y(x)}dx + \epsilon_{\gamma,k}.
\end{align}

Finally using \eqref{bias_proof_2} and \eqref{bias_proof_3} in \eqref{bias_proof_1} results in 

\begin{align}
\mathbb{B}\of[{\widehat{D_p}(X,Y)}]=\epsilon_{\gamma,k}+ O\of{\frac{1}{N}}=\epsilon_{\gamma,k}.
\end{align}

\end{proof}



\begin{proof}[Proof of Theorem \ref{variance}]

Without loss of generality, assume that $N < M$. Introduce extra virtual random points $X_{N+1},...,X_M$ with distribution $f_X(x)$ and define $Z_i:=(X_i,Y_i)$. We use the Efron\hyp Stein inequality on $Z:=(Z_1,...,Z_M)$. Let $Z':=(Z'_1,...,Z'_M)$ be another independent copy of $Z$. Further define $Z^{(i)}:=(Z_1,...,Z_{i-1},Z'_i,Z_{i+1},...,Z_M)$. We also use the simpler notations of $\widehat{D_p}(Z):=\widehat{D_p}(X,Y)$ and $\mathcal{E}(Z):=\mathcal{E}(X,Y)$. Then, according to Efron\hyp Stein inequality we have

\begin{align} \label{Var_proof_1}
\mathbb{V}\of[\widehat{D_p}(Z)] &\leq \frac{1}{2} \sum_{i=1}^M \mathbb{E}\left[\left(\widehat{D_p}(Z)-\widehat{D_p}(Z^{(i)})\right)^2\right]\nonumber\\
&= \frac{1}{2} \sum_{i=1}^N \mathbb{E}\left[\left(\widehat{D_p}(Z)-\widehat{D_p}(Z^{(i)})\right)^2\right]\nonumber\\
&\qquad + \frac{1}{2} \sum_{i=N+1}^M \mathbb{E}\left[\left(\widehat{D_p}(Z)-\widehat{D_p}(Z^{(i)})\right)^2\right].
\end{align}

We first find a bound on the first term using the definition in \eqref{define_knn_est}:

\begin{align}
&\frac{1}{2} \sum_{i=1}^N \mathbb{E}\left[\left(\widehat{D_p}(Z)-\widehat{D_p}(Z^{(i)})\right)^2\right]=\nonumber\\
&\qquad=\half\of{\frac{N+M}{2NM}}^2\sum_{i=1}^N \mathbb{E}\left[\left(|\mathcal{E}(Z)|-|\mathcal{E}(Z^{(i)})|\right)^2\right]\nonumber\\
&\qquad\leq\half\of{\frac{N+M}{2NM}}^2\sum_{i=1}^N 4\nonumber\\
&\qquad\leq\frac{\of{N+M}^2}{2NM^2}\nonumber\\
&\qquad=O\of{\frac{1}{N}},
\end{align}
where in the third line we have used the fact that resampling $Z_i$ at most changes $|\mathcal{E}(Z)|$ by two. We can use the same argument for the second term in \eqref{Var_proof_1} with the difference that resampling $Z_i$ for $N+1\leq i\leq M$ changes $|\mathcal{E}(Z)|$ at most by one. So we get a similar bound of 
\begin{align}
&\frac{1}{2} \sum_{i=N+1}^M \mathbb{E}\left[\left(\widehat{D_p}(Z)-\widehat{D_p}(Z^{(i)})\right)^2\right]=O\of{\frac{1}{M}}.
\end{align}
Note that since $M=\floor{\frac{Nq}{p}}$, $O\of{\frac{1}{M}}=O\of{\frac{1}{N}}$. Therefore we have

\begin{align}
\mathbb{V}\of[\widehat{D_p}(Z)] &\leq O\of{\frac{1}{N}}.
\end{align}
\end{proof}



\begin{proof}[Proof of Theorem \ref{WNN_Theorem}]

The key to proving the theorem on the MSE of the WNN estimator is to use the theory of optimally weighted ensemble estimation \cite{Kevin16,structure2016}. Note that estimator can be written as a weighted sum of $k$-NN estimators with different parameters:

\begin{align}\label{ensemble_form}
\widehat{D}_p^W(X,Y)&=1-|\mathcal{E}_K^W(X,Y)|\frac{N+M}{2NM}\nonumber\\
&=1-\sum_{l\in\mathcal{L}}W(l)|\mathcal{E}_l(X,Y)|\frac{N+M}{2NM}\nonumber\\
&=\sum_{l\in\mathcal{L}}W(l)\of{1-|\mathcal{E}_l(X,Y)|\frac{N+M}{2NM}}\nonumber\\
&=\sum_{l\in\mathcal{L}}W(l)\widehat{D}^l_p(X,Y),
\end{align}
where $\mathcal{E}_k(X,Y)$ are the sets of edges of $k$-NN graph, which connect nodes with different types, and $\widehat{D}^l_p(X,Y)$ is the corresponding estimator. In order to prove the MSE rate of ensemble estimator, we need to derive the bias term in more accurate form than Theorem \ref{bias_theorem}. Note than more restrictive assumptions are required to derive the accurate bias rate. The density functions are assumed to be in the H\"{o}lder space $\Sigma(\gamma,B)$. Fix a constant parameter $\kappa\in \mathbb{N}$, and define the interior support, denoted by $\mathcal{X}_I^\kappa$, as the set of all points for which the $\kappa$-NN ball is completely in the boundary and the support boundary, denoted by $\mathcal{X}_B^\kappa$, is defined as the set of all points for which the $\kappa$-NN ball meets the boundary. According to \cite{noshad} (equation (63)-(64)), for the interior points we have 

\begin{align}\label{interior}
&\Pr\of{Q_k(X_1)\in Y}=\frac{f_Y(X_1)}{\eta^{-1}f_X(X_1)+ f_Y(X_1)}\nonumber\\
&\qquad+
\sum_{i=1}^{q-1} a_i(X_1)(k/N)^{i/d}+\epsilon_{\gamma,k}, 
\end{align}
where $a_i(X_1)$ is a function depending only on $X_1$. Also for the boundary points we have

\begin{align}\label{exterior}
\Pr\of{Q_k(X_1)\in Y}=\frac{f_Y(X_1)}{\eta^{-1}f_X(X_1)+ f_Y(X_1)}+\epsilon_{\gamma,\kappa}.
\end{align}

Let define the notations $P_{\kappa,N}:=\Pr\of{X_1\in \mathcal{X}_I^\kappa}$, $\overline{P}_{\kappa,N}:=1-P_{\kappa,N}$ and $\phi_{i,\kappa}(N):= P_{\kappa,N} \E{a_i(X_1)\vert X_1\in \mathcal{X}_I^\kappa}/2q$. From \eqref{bias_proof_2} and using the equations \eqref{interior} and \eqref{exterior}, we get 

\begin{align}\label{bias_proof_2}
\frac{\E{\sum_{i=1}^{N}N_i}}{2Nq}
&=\frac{\mathbb{E}\of[\Pr\of{Q_k(X_1)\in Y}]}{2q}\nonumber\\
&=\frac{P_{\kappa,N}\mathbb{E}\of[\Pr\of{Q_k(X_1)\in Y}\vert X_1\in \mathcal{X}_I^\kappa]}{2q}\nonumber\\
&+\frac{\overline{P}_{\kappa,N}\mathbb{E}\of[\Pr\of{Q_k(X_1)\in Y}\vert X_1\notin \mathcal{X}_I^\kappa]}{2q}\nonumber\\
&=\half\int\frac{f_X(x)f_Y(x)}{pf_X(x)+qf_Y(x)}dx \nonumber\\
&\qquad +\sum_{i=1}^{q-1} \phi_{i,\kappa}(N)(\frac{k}{N})^{i/d}+\epsilon_{\gamma,\kappa},
\end{align}

and from \eqref{bias_proof_1}, and the choice of $\kappa=\floor{c\sqrt{N}}$ we get

\begin{align}
\mathbb{B}\of[{\widehat{D}_p^k(X,Y)}]=\sum_{i=1}^{q-1} \phi_{i,\kappa}(N)(\frac{k}{N})^{i/d}+O\of{\of{\frac{1}{N}}^{q/2}}.
\end{align}

The proof follows by using the ensemble theorem in (\cite{Kevin16}, Theorem 4) with the parameters $\psi_i(l)=l^{i/d}$ and $\phi'_{i,d}(N)=\phi_{i,\kappa}(N)/N^{i/d}$. Using this theorem and equation \eqref{ensemble_form}, the MSE rate $\widehat{D}_p^W(X,Y)$ is $O\of{\frac{1}{N}}$.

\end{proof}


\end{document}